\documentclass[lettersize,journal]{IEEEtran}

\usepackage[utf8]{inputenc}
\usepackage[T1]{fontenc}
\usepackage{amsmath,amssymb,amsthm,mathtools}
\usepackage{bm}
\usepackage{graphicx}
\usepackage{booktabs}
\usepackage{array}
\usepackage{multirow}
\usepackage{xcolor}
\usepackage{enumitem}
\usepackage{cite}
\usepackage{stfloats}
\usepackage{algorithm}
\usepackage{algpseudocode}
\usepackage[colorlinks=true,linkcolor=blue!60!black,citecolor=blue!60!black,urlcolor=blue!60!black]{hyperref}
\usepackage{cleveref}
\crefname{figure}{Fig.}{Figs.}
\Crefname{figure}{Fig.}{Figs.}

\newtheoremstyle{paper-plain}%
  {1.0\baselineskip}{0.5\baselineskip}%
  {\itshape}{}{\bfseries}{.}{0.5em}{}
\theoremstyle{paper-plain}
\newtheorem{theorem}{Theorem}
\newtheorem{proposition}[theorem]{Proposition}
\newtheorem{lemma}[theorem]{Lemma}

\theoremstyle{definition}
\newtheorem{definition}[theorem]{Definition}

\usepackage{xcolor}

\newcommand{\mat}[1]{\mathbf{#1}}
\DeclareMathAlphabet\mathbfcal{OMS}{cmsy}{b}{n}

\begin{document}

\title{DeepOHeat-v2: Self-Improving Operator Learning for\\ Fast and Trustworthy
       Thermal Optimization in 3D-IC Design}

\author{Xinling Yu, Yixing Li, Ziyue Liu, Xin Ai, Zhiyu Zeng, Hai Li, and Zheng Zhang%
\thanks{Xinling Yu and Zheng Zhang are with the Department of Electrical and Computer Engineering, University of California, Santa Barbara, CA 93106 USA (e-mail: xyu644@ucsb.edu; zhengzhang@ece.ucsb.edu).}%
\thanks{Ziyue Liu is with the Department of Computer Science, University of California, Santa Barbara, CA 93106 USA (e-mail: ziyueliu@ucsb.edu).}%
\thanks{Yixing Li, Xin Ai, and Zhiyu Zeng are with Cadence Design Systems, Austin, TX 78759 USA (e-mail: yixingli@cadence.com; nathanai@cadence.com; zzeng@cadence.com).}%
\thanks{Hai Li is an independent researcher (e-mail: hai.li.pnw@gmail.com).}}


\maketitle

\begin{abstract}
Thermal-aware optimization of multi-die 3D integrated circuits evaluates many
designs, each a costly heat-equation solve. Operator-learning surrogates replace
this solve with a fast forward pass, ideally trained from physics alone, without
labeled data. DeepOHeat-v1 made such surrogates fast and trustworthy, but only on
low-contrast geometries. High-contrast multi-die stacks break it in two ways:
discontinuous conductivities make the continuous physics loss ill-defined at
material interfaces, and ill-conditioning
($\kappa_2(\mat{A}_h)\!\approx\!6\!\times\!10^4$) puts the discretized strong-form
loss beyond first-order optimization. We propose DeepOHeat-v2 to overcome both.
First, we train on a discretized physics loss that handles the discontinuities
natively; its energy form reduces the prediction-space loss-Hessian conditioning
from $\kappa^2$ to $\kappa$, and a matrix-preconditioned optimizer cuts the mean
peak temperature error from over $30$~K to $0.55$~K. Second, because optimization leaves the training
distribution, we propose a self-improving framework: a hotspot trust gate sends flagged placements to a
reference solver, and the surrogate incrementally retrains on the refined
solutions, keeping an update only when it improves held-out validation error. On a multi-die benchmark,
the surrogate--true peak gap on the returned design falls from $1.12$~K to
$0.11$~K, matching a solve-at-every-step optimizer while running $56\times$ faster. 
\end{abstract}

\begin{IEEEkeywords}
3D-IC thermal analysis, thermal optimization, operator learning,
physics-informed learning.
\end{IEEEkeywords}

\section{Introduction}
\label{sec:intro}

\IEEEPARstart{T}{hree}-dimensional integrated circuits (3D-ICs) stack multiple active dies to achieve integration densities and performance beyond the limits of planar scaling. However, this vertical integration also concentrates power within a small volume and increases the thermal path length, making heat dissipation far harder~\cite{tavakkoli2016thermal,cao2019survey,iyer2016heterogeneous}. The
resulting temperature rise and on-chip thermal gradients degrade
timing, reliability, and lifetime and increase leakage, so thermal analysis and
thermal-aware optimization have become standard steps of the 3D-IC design
flow. Both reduce to the same computation: predicting the temperature profile of a
candidate design, once for every design the optimizer visits. Traditional high-fidelity numerical
solvers (e.g., finite element and finite difference methods) of the heat equation~\cite{li2004fullchip,liu2014tsv,sultan2019survey}
are accurate but far too slow to call repeatedly inside such a loop; the thermal
solver is the computational bottleneck, and a fast, accurate surrogate is the way
around it.

Neural-network surrogates can predict temperature fields in milliseconds once
trained. Early data-driven models learn this mapping from large sets of
solver-generated
examples~\cite{ranade2022thermal,wen2020dnn,smith2023realtime,yang2026real}, but
their accuracy depends on how well the training set covers the design space, and
generating that set is often computationally prohibitive. Operator
learning~\cite{lu2021deeponet,li2021fno,yu2024separable,wu2024transolver}
provides a more general formulation: a neural operator approximates the solution
map from a design configuration to the corresponding temperature field, so a
new design requires only a forward pass. An operator network can also be
trained \emph{from physics alone}, using a physics-informed
loss~\cite{raissi2019pinn} that supervises it directly from the governing heat
equation and removes the need for labeled solver
data~\cite{lu2024fast,wu2026data}. The DeepOHeat
series~\cite{liu2023deepoheat,yu2025deepoheat} brought this data-free,
physics-informed operator-learning approach to 3D-IC thermal analysis. Its most recent version, DeepOHeat-v1, makes this
practical with separable training~\cite{yu2024separable} and Kolmogorov--Arnold trunks~\cite{liu2024kan} for efficiency and accuracy, and
a confidence-gated solver refinement for trustworthiness, but has so far been
demonstrated only on mild, low-contrast geometries.

The first challenge is training itself: realistic multi-die stacks break this
data-free approach on two counts. Their conductivities are discontinuous, changing abruptly both between
material layers and at the walls of the copper through-silicon vias (TSVs) that
pass through the dies; the continuous strong-form physics-informed loss used by DeepOHeat-v1 then cannot represent the
temperature field across such interfaces. PI-ONet~\cite{sha2025pi} instead adds explicit
interface-continuity terms to the loss and homogenizes each TSV into an
equivalent-conductivity block; but a realistic stack has many interfaces, so the
number of interface terms grows quickly, and every added loss term needs a
carefully tuned weight to train stably~\cite{wang2021pinn}. The conductivity contrast also makes the discretized operator ill-conditioned,
and the squared strong-form residual loss inherits the \emph{square} of that
conditioning; first-order optimization then effectively cannot minimize
it~\cite{deryck2024operator,rathore2024challenges}. Data-driven operators
such as ARO~\cite{wang2024aro} avoid both issues but require large labeled
datasets, as does PI-ONet's hybrid loss. No existing method is data-free
and robust to both problems.

The second challenge appears once a trained surrogate is placed inside the
optimization loop. The surrogate is pretrained on one distribution of randomly generated designs,
but the optimizer queries it on the evolving designs it visits during the
search, where the surrogate has seen little data and can be inaccurate.
DeepOHeat-v1 mitigates this with a confidence-gated hybrid scheme: a
residual-based check accepts a prediction when it appears trustworthy and
otherwise refines it with a warm-started generalized minimal residual (GMRES)
solve~\cite{saad1986gmres}, recovering solver-level accuracy at a fraction of
the cost. DeepOHeat-v1's scheme keeps each evaluation accurate, but the surrogate itself
{\it never improves}: every difficult design must be re-solved, and on such a stack
the confidence signal that decides when to refine becomes unreliable. The resulting refined solutions
 are the high-fidelity,
on-distribution data the surrogate lacks, yet DeepOHeat-v1 discards them.
Our self-improving framework reuses these verified solutions as training data: the
surrogate improves on the designs the search visits, and the loop calls the solver
{\it progressively less}.

To address these challenges, we develop \emph{DeepOHeat-v2}. Its main
contributions are:
\begin{enumerate}[leftmargin=*]
\item \textbf{Data-free physics-informed training for heterogeneous stacks.} We
train the surrogate on a \emph{discretized} physics loss, which handles the
discontinuous conductivities of the layers and TSVs natively, with no
interface-condition terms and no TSV homogenization. Casting this loss in an
\emph{energy form}~\cite{e2018deepritz} provably reduces its prediction-space
loss-Hessian condition number from $\kappa^2$ to $\kappa$. Since this $\kappa$ alone still
stalls first-order optimization, we minimize the energy loss with a
matrix-preconditioned optimizer~\cite{liu2026muon} that rescales each update to
counteract the residual ill-conditioning.

\item \textbf{A self-improving operator-learning framework.} We close a loop
between optimization and training: when the trust gate flags a prediction during
the search, the solver refines it; the surrogate trains incrementally on these
verified in-trajectory solutions; and a held-out check decides which model to keep
(pre-update, updated, or weight-averaged~\cite{wortsman2022soup}). This makes the surrogate
accurate where the search operates, and because it keeps improving along the
trajectory, the loop calls the reference solver progressively less as the run
proceeds. 

\item \textbf{Validation on a high-contrast 3D-IC benchmark.} On a face-to-face
chiplet stack with an $800\times$ conductivity contrast, data-free training reaches a
mean peak temperature error of $0.55$~K, versus $>\!30$~K for the strong-form loss;
the hotspot trust gate catches $64\%$ and $92\%$ of the $>\!0.5$~K and $>\!1$~K
prediction errors at a $30\%$ flag rate, where the DeepOHeat-v1 residual criterion is
no better than random; and the self-improving loop attains oracle placement quality
($0.11$~K surrogate--true gap) in $292$~s, $56\times$ faster than solving at every step
and with about half the solver calls of the non-adaptive DeepOHeat-v1 loop. 
\end{enumerate}

\section{Background and Motivation}
\label{sec:background}

\subsection{Thermal-Aware Design Optimization}
\label{sec:obj}

Thermal-aware design optimization seeks a chip design whose hottest junction is
as cool as possible. A design configuration $\mat{u} \in \mathcal{U}$ gathers the
tunable parameters that shape a chip's heat generation and conduction, such as
block placement, per-block power, or layer materials. It induces a temperature field $T_{\mat{u}}$, whose peak we denote $T_{\mathrm{peak}}(\mat{u})$. The problem is
\begin{align}
\label{eq:obj}
\min_{\mat{u} \in \mathcal{U}} T_{\mathrm{peak}}(\mat{u})
\quad \text{subject to } \mat{u} \text{ feasible}.
\end{align}
Feasibility encodes the design rules of the instance, such as non-overlapping,
in-bounds blocks. Problem \eqref{eq:obj} is solved
by an outer-loop optimizer that proposes configurations and evaluates
$T_{\mathrm{peak}}(\mat{u})$ at each step. Every evaluation requires a full thermal solve, so
the solver dominates the optimization cost.

Evaluating $T_{\mathrm{peak}}(\mat{u})$ requires solving the heat equation on
the chip domain $\Omega \subset \mathbb{R}^3$. With spatial coordinate $\mat{y}$,
thermal conductivity $k(\mat{y})$ that is piecewise constant across material
regions, and a configuration-dependent volumetric heat source
$q_V(\mat{y}, \mat{u})$, the temperature field satisfies
\begin{align}
\label{eq:strong}
\nabla \cdot \bigl(k(\mat{y})\,\nabla T\bigr) + q_V(\mat{y}, \mat{u}) = 0,
\quad \mat{y} \in \Omega ,
\end{align}
with Robin convection on the top and bottom surfaces and adiabatic side walls,
\begin{equation}
\label{eq:bc}
\begin{aligned}
-k(\mat{y})\,\partial_n T &= h\,(T - T_{\mathrm{amb}})
   && \text{on } \Gamma_{\mathrm{top}},\, \Gamma_{\mathrm{bot}}, \\
\partial_n T &= 0 && \text{on } \Gamma_{\mathrm{side}} .
\end{aligned}
\end{equation}
Here $\partial_n$ is the outward-normal derivative, $T_{\mathrm{amb}}$ the ambient
temperature, and $h$ the convection coefficient. Because $k$ is discontinuous
between material regions, Eq.~\eqref{eq:strong} holds within each region and is
completed by continuity of the temperature and of the normal heat flux
$k\,\partial_n T$ across their interfaces. Discretizing
Eq.~\eqref{eq:strong}--\eqref{eq:bc} via classical methods like finite difference or finite element yields a sparse linear system
\begin{align}
\label{eq:fvm}
\mat{A}_h\,\mat{T} = \mat{b}_h .
\end{align}
Each thermal evaluation in Eq.~\eqref{eq:obj} is one solve of Eq.~\eqref{eq:fvm}. On a
high-resolution mesh, this per-evaluation cost is what a surrogate must replace.

\subsection{The DeepOHeat-v1 Framework}
\label{sec:v1}

DeepOHeat~\cite{liu2023deepoheat} casts thermal simulation as an
operator-learning problem, and DeepOHeat-v1~\cite{yu2025deepoheat} makes the
resulting surrogate fast enough for full-chip resolution and trustworthy enough
to drive an optimizer. We review the formulation, architecture, data-free training objective, and hybrid optimization loop on which DeepOHeat-v2 builds.

\subsubsection{Operator Learning for Thermal Simulation}
For every configuration $\mat{u} \in \mathcal{U}$, the boundary-value problem
\eqref{eq:strong}--\eqref{eq:bc} has a unique temperature field
$T_{\mat{u}} \in \mathcal{S}$. The map that assigns this field to each
configuration is the \emph{solution operator}
\begin{equation}
\label{eq:operator}
G : \mathcal{U} \to \mathcal{S}, \qquad G(\mat{u}) = T_{\mat{u}} .
\end{equation}
Operator learning~\cite{lu2021deeponet,li2021fno} approximates $G$ by a neural surrogate $G_{\boldsymbol{\theta}}$ with parameters $\boldsymbol{\theta}$, trained once so a new configuration's temperature field follows from a single forward pass rather than a fresh solve of Eq.~\eqref{eq:fvm}.

\subsubsection{Separable Operator Network with ChebyKAN Trunks}
The original DeepOHeat~\cite{liu2023deepoheat} represents the surrogate $G_{\boldsymbol{\theta}}$ as
a DeepONet~\cite{lu2021deeponet}. A \emph{branch} network encodes a configuration
$\mat{u}$, such as its discretized power map $\mat{q}(\mat{u})$, into $r$
coefficients $\beta_1(\mat{u}),\dots,\beta_r(\mat{u})$; a \emph{trunk} network maps a
spatial coordinate $\mat{y}$ to $r$ basis functions
$\tau_1(\mat{y}),\dots,\tau_r(\mat{y})$; and the predicted temperature field is
their inner product,
\begin{equation}
\label{eq:deeponet}
G_{\boldsymbol{\theta}}(\mat{u})(\mat{y}) = \sum_{j=1}^{r}\beta_j(\mat{u})\,\tau_j(\mat{y}) .
\end{equation}
DeepOHeat-v1~\cite{yu2025deepoheat} improves this architecture along two axes. For
\emph{scalability}, it makes the trunk \emph{separable}~\cite{yu2024separable},
replacing the single trunk over $\mat{y}=(y_1,y_2,y_3)$ with one trunk per
spatial axis:
\begin{equation}
\label{eq:separable}
G_{\boldsymbol{\theta}}(\mat{u})(y_1, y_2, y_3)
  = \sum_{j=1}^{r} \beta_j(\mat{u})\,
    \tau^1_j(y_1)\,\tau^2_j(y_2)\,\tau^3_j(y_3) .
\end{equation}
A full-grid prediction then requires only $N_1+N_2+N_3$ trunk evaluations and an
outer product, rather than $N_1 N_2 N_3$ pointwise evaluations; combined with
forward-mode differentiation~\cite{khan2015vector} of the physics loss, this reduction makes
physics-informed training feasible at full-chip resolution. For \emph{accuracy},
the trunks are Chebyshev Kolmogorov--Arnold networks
(ChebyKANs)~\cite{liu2024kan, ss2024chebyshev}, whose learnable univariate edge functions mitigate
the spectral bias of plain multilayer-perceptron (MLP) trunks and resolve the
steep temperature gradients near hotspots; the branch remains an MLP.

\subsubsection{Data-Free Physics-Informed Training}
\label{sec:v1-loss}
The surrogate is trained directly from the governing equations, without using
simulation data. Define the interior residual of a
candidate field $T$ for configuration $\mat{u}$,
\begin{equation}
\label{eq:residual}
\mathcal{R}[T](\mat{y};\mat{u}) :=
  \nabla\!\cdot\!\bigl(k(\mat{y})\,\nabla T(\mat{y})\bigr) + q_V(\mat{y},\mat{u}),
\end{equation}
and let $\mathcal{B}[\,\cdot\,]$ denote the boundary operator collecting the
conditions in Eq.~\eqref{eq:bc}. Given training configurations
$\{\mat{u}^{(i)}\}_{i=1}^{N_u}\!\subset\!\mathcal{U}$, interior collocation points
$\{\mat{y}_r^{(j)}\}_{j=1}^{N_r}\!\subset\!\Omega$, and boundary points
$\{\mat{y}_b^{(j)}\}_{j=1}^{N_b}\!\subset\!\partial\Omega$, DeepOHeat-v1
minimizes the continuous physics-informed loss
\begin{equation}
\label{eq:v1loss}
\mathcal{L}(\boldsymbol{\theta}) = \mathcal{L}_{\Omega}(\boldsymbol{\theta})
  + \lambda_b\,\mathcal{L}_{\Gamma}(\boldsymbol{\theta}),
\end{equation}
whose interior and boundary terms penalize the residuals of the predicted field
$G_{\boldsymbol{\theta}}(\mat{u}^{(i)})$ at the collocation points,
\begin{align}
\label{eq:v1loss-terms}
\mathcal{L}_{\Omega}(\boldsymbol{\theta})
  &= \frac{1}{N_u N_r}\sum_{i=1}^{N_u}\sum_{j=1}^{N_r}
     \mathcal{R}\bigl[G_{\boldsymbol{\theta}}(\mat{u}^{(i)})\bigr]\!\bigl(\mat{y}_r^{(j)};\mat{u}^{(i)}\bigr)^2 ,
   \nonumber \\
\mathcal{L}_{\Gamma}(\boldsymbol{\theta})
  &= \frac{1}{N_u N_b}\sum_{i=1}^{N_u}\sum_{j=1}^{N_b}
     \mathcal{B}\bigl[G_{\boldsymbol{\theta}}(\mat{u}^{(i)})\bigr]\!\bigl(\mat{y}_b^{(j)}\bigr)^2 ,
\end{align}
where $\lambda_b>0$ balances the two terms. All derivatives act on the network
output through automatic differentiation, so training never requires labeled
solver data, which is why DeepOHeat-v1 is data-free.

\subsubsection{Hybrid Optimization for Trustworthiness}
A surrogate is fast but not guaranteed accurate on every configuration an
optimizer visits, so DeepOHeat-v1 pairs it with a verify-and-refine loop. Sampling
the predicted field $G_{\boldsymbol{\theta}}(\mat{u})$ on the mesh gives a vector $\widehat{\mat{T}}$, whose
relative residual against the discretized system,
\begin{equation}
\label{eq:relres}
r_{\mathrm{rel}} = \frac{\|\mat{A}_h\widehat{\mat{T}} - \mat{b}_h\|_2}{\|\mat{b}_h\|_2} ,
\end{equation}
costs a single sparse matrix--vector product. The loop compares $r_{\mathrm{rel}}$ against a threshold. If the residual is small, the
prediction is accepted; otherwise a refinement solves Eq.~\eqref{eq:fvm} by
GMRES~\cite{saad1986gmres} warm-started from $\widehat{\mat{T}}$; this recovers a
near-solver field in far fewer iterations than a cold solve. 

\subsection{Challenges of High-Contrast Multi-Die Stacks}
\label{sec:benchmark}
DeepOHeat-v1 was demonstrated on mild, low-contrast geometries. Real multi-die
3D-ICs are instead dominated by high material contrast: stacked dies, copper TSVs, and an
organic substrate place conductivities orders of magnitude apart within a single
chip, an $\sim\!800\times$ range on our benchmark stack (\Cref{sec:benchmark-stack}).

At this contrast, three new challenges arise that DeepOHeat-v1 cannot handle. First, the continuous physics loss \eqref{eq:v1loss} of DeepOHeat-v1 cannot represent the field across the
discontinuous layer and TSV interfaces, where $k$ changes abruptly and the
temperature is only $C^0$ (continuous but not differentiable). Second, the squared residual inherits the \emph{square} of the operator
conditioning ($\kappa^2\!\approx\!3.6\!\times\!10^9$ on our benchmark), which puts it
beyond first-order optimization. Finally, the hybrid loop discards each solver-refined field after one use, so the
surrogate never improves and the same solver work is repeated every time. 

In this paper, DeepOHeat-v2's training method (\Cref{sec:pretrain}) fixes the first two; its
self-improving framework (\Cref{sec:pipeline}) addresses the third challenge.

\section{Data-Free Physics-Informed Training on Heterogeneous Stacks}
\label{sec:pretrain}

Training the surrogate from physics alone on a heterogeneous stack faces two
obstacles: the loss must represent the temperature field across discontinuous
material interfaces, and it must stay optimizable despite the operator's
ill-conditioning. We address them in three steps. Discretizing by the finite-volume
method (FVM) handles the interfaces by construction (\Cref{sec:fvm-loss}). Writing
the resulting loss in energy form reduces the prediction-space loss-Hessian
conditioning from $\kappa^2$ to $\kappa := \kappa_2(\mat{A}_h)$
(\Cref{thm:energy-kappa}). The remaining factor $\kappa$ still stalls a per-coordinate optimizer, because at high contrast the stiff and compliant curvature directions couple across parameters; we remove it with a matrix-valued preconditioner that rescales them jointly (\Cref{sec:muon2}).

\subsection{An FVM-Discretized Physics Loss}
\label{sec:fvm-loss}

\paragraph*{Why not the continuous strong form}
DeepOHeat-v1 represents the field by a smooth network $T_{\boldsymbol{\theta}} : \Omega \to \mathbb{R}$ and
enforces the strong form \eqref{eq:strong} at collocation points. A discontinuous $k$
breaks this at the level of the formulation, not the optimization. Write
$\Gamma$ for a material interface, $T^{\pm}$ and $\partial_n T^{\pm}$ for the field and
its normal derivative on either side, and $[\![f]\!]_\Gamma := f^+ - f^-$ for the abrupt change in $f$ across $\Gamma$.

\begin{proposition}[The strong-form residual is singular at interfaces]
\label{prop:singular}
Let $k$ be piecewise constant with an abrupt change across $\Gamma$, and let
$T_{\boldsymbol{\theta}} \in C^1(\Omega)$ (continuous, with continuous first derivatives). Then,
in the distributional sense,
\begin{equation}
\label{eq:singular}
\nabla\!\cdot\!\bigl(k\nabla T_{\boldsymbol{\theta}}\bigr)
  = k\,\Delta T_{\boldsymbol{\theta}}
  + (k^+\!-\!k^-)\,\partial_n T_{\boldsymbol{\theta}}\;\delta_\Gamma ,
\end{equation}
where $\Delta$ is the piecewise Laplacian and $\delta_\Gamma$ is the surface
Dirac measure on $\Gamma$. The interface term has no finite value unless
$\partial_n T_{\boldsymbol{\theta}} = 0$, so driving the strong-form residual to zero forces a
smooth network to carry \emph{no} normal flux across every interface, incompatible
with conducting heat.
\end{proposition}

\begin{proposition}[The physical field has an interface kink]
\label{prop:kink}
The physical temperature is continuous, $[\![T]\!]_\Gamma = 0$, but continuity of
the normal heat flux $k\,\partial_n T$ forces an abrupt slope change,
\begin{equation}
\label{eq:kink}
\partial_n T^- = \frac{k^+}{k^-}\,\partial_n T^+ \ne \partial_n T^+
\qquad (k^+ \ne k^-) .
\end{equation}
Hence $T$ is $C^0$ (continuous) but not $C^1$ at $\Gamma$, and lies outside the
hypothesis class of any $C^1$ network.
\end{proposition}

Together the propositions show that the problem lies in the model class, not the
optimization: a smooth network cannot both
carry heat across an interface and drive the strong-form residual to zero because
the field it must fit is itself non-smooth. Domain-decomposition
physics-informed neural networks (PINNs)~\cite{jagtap2020cpinn} and discontinuity-capturing architectures~\cite{hu2022dcsnn}
recover such fields, but need the interface geometry known in advance; a multi-die stack carries thousands of TSV interfaces whose
layout changes with every pattern, so neither scales.

\paragraph*{The discretized loss}
We instead let the discretization carry the interface physics. Integrating
Eq.~\eqref{eq:strong} over each control volume with harmonic-mean face
conductivities~\cite{patankar1980numerical} gives the symmetric positive-definite
(SPD) system $\mat{A}_h\mat{T}=\mat{b}_h$ of \eqref{eq:fvm}; the harmonic mean enforces flux
continuity at every face, so the discontinuities of
Propositions~\ref{prop:singular}--\ref{prop:kink} are handled by construction. We
train on the data-free squared residual of this system,
\begin{equation}
\label{eq:strongloss}
\mathcal{L}_{\mathrm{strong}}(\boldsymbol\theta) = \|\mat{A}_h\widehat{\mat{T}}-\mat{b}_h\|_2^2 ,
\end{equation}
assembled from the coefficients of $\mat{A}_h$ and $\mat{b}_h$ with no solved field. Its
gradient is a fixed stencil applied through $\mat{A}_h$, avoiding the second-order
automatic differentiation a continuous loss would need at every collocation point on
the $N\!\approx\!2\!\times\!10^6$-cell grid. Discretized residuals also appear in~\cite{gao2021phygeonet,chiu2022canpinn},
there for smooth geometries.

\paragraph*{Conditioning of the discretized loss}
That settles the representation problem, but $\mat{A}_h$ is ill-conditioned: $\kappa_2(\mat{A}_h)$ grows
with the conductivity ratio $k_{\max}/k_{\min}$ independently of the scheme
(Appendix~\ref{app:continuous}), reaching $\kappa_2(\mat{A}_h)\approx6.02\!\times\!10^4$ on our
benchmark (\Cref{sec:benchmark-stack}). The squared residual \eqref{eq:strongloss} then
\emph{squares} it: its Hessian $2\mat{A}_h^2$ has condition number
$\kappa_2(\mat{A}_h)^2\approx3.6\!\times\!10^9$, putting the loss beyond first-order
optimization (\Cref{tab:pretrain}).

\subsection{A Reformulation with the Energy Form}
\label{sec:energy-form}

To remove the squaring, we replace the squared residual \eqref{eq:strongloss}
with the energy form of the same linear system.

\begin{definition}[Energy-form discrete loss]
\label{def:energy}
For a predicted vector $\widehat{\mat{T}} \in \mathbb{R}^N$, the FVM energy-form loss is
\begin{equation}
\label{eq:energyloss}
\mathcal{L}_{\mathrm{energy}}(\boldsymbol\theta) := \tfrac{1}{2}\,\widehat{\mat{T}}^\top\mat{A}_h\,\widehat{\mat{T}} - \mat{b}_h^\top\,\widehat{\mat{T}} .
\end{equation}
\end{definition}

\begin{theorem}[Energy form reduces the conditioning from $\kappa^2$ to $\kappa$]
\label{thm:energy-kappa}
The strong-form loss \eqref{eq:strongloss} and the energy-form loss
\eqref{eq:energyloss} share the unique minimizer $\mat{T}^* = \mat{A}_h^{-1}\mat{b}_h$, but as
functions of the prediction $\widehat{\mat{T}}$ their Hessians have different
conditioning,
\begin{align}
\kappa(\nabla_{\widehat{\mat{T}}}^2\mathcal{L}_{\mathrm{strong}}) = \kappa_2(\mat{A}_h)^2, \qquad
\kappa(\nabla_{\widehat{\mat{T}}}^2\mathcal{L}_{\mathrm{energy}}) = \kappa_2(\mat{A}_h).
\end{align}
\end{theorem}

\begin{proof}
Both losses are quadratic in the prediction; write $\mat{e} := \widehat{\mat{T}}-\mat{T}^*$
with $\mat{T}^* = \mat{A}_h^{-1}\mat{b}_h$. Then
\begin{align}
\mathcal{L}_{\mathrm{strong}}(\widehat{\mat{T}}) &= \|\mat{A}_h\mat{e}\|_2^2 = \mat{e}^\top \mat{A}_h^2\,\mat{e},
  & \nabla_{\widehat{\mat{T}}}^2\mathcal{L}_{\mathrm{strong}} &= 2\mat{A}_h^2 , \\
\mathcal{L}_{\mathrm{energy}}(\widehat{\mat{T}}) &= \tfrac{1}{2}\,\mat{e}^\top \mat{A}_h\,\mat{e} + \text{const},
  & \nabla_{\widehat{\mat{T}}}^2\mathcal{L}_{\mathrm{energy}} &= \mat{A}_h .
\end{align}
Because $\mat{A}_h$ is SPD, both Hessians are SPD, so $\mat{T}^*$ is the unique minimizer
of each loss. Their condition numbers are $\kappa_2(\mat{A}_h^2)=\kappa_2(\mat{A}_h)^2$ and
$\kappa_2(\mat{A}_h)$, respectively.
\end{proof}

$\mathcal{L}_{\mathrm{energy}}$ is the discrete energy of the boundary-value problem (Dirichlet
energy plus Robin boundary term); minimizing this strictly convex quadratic solves
$\mat{A}_h\mat{T}=\mat{b}_h$ with curvature $\mat{A}_h$ rather than $\mat{A}_h^2$, making first-order
operator learning feasible at high contrast. Variational energies are an established training objective in physics-informed learning~\cite{e2018deepritz,kharazmi2019vpinn}; we bring this form to the discrete finite-volume operator on a discontinuous stack and quantify its conditioning gain.

\Cref{thm:energy-kappa} concerns the Hessian in the prediction $\widehat{\mat{T}}$, but
training optimizes the network parameters $\boldsymbol\theta$, and the energy loss
$\mathcal{L}(\boldsymbol\theta)=\mathcal{L}_{\mathrm{energy}}(\widehat{\mat{T}}(\boldsymbol\theta))$ is nonconvex in $\boldsymbol\theta$. With
$\mat{J} := \partial\widehat{\mat{T}}/\partial\boldsymbol\theta$ the output Jacobian and residual
$\mat{r} := \mat{A}_h\widehat{\mat{T}}-\mat{b}_h = \nabla_{\widehat{\mat{T}}}\mathcal{L}_{\mathrm{energy}}$, the parameter-space
Hessian splits into a Gauss--Newton term and a residual-weighted term,
\begin{equation}
\label{eq:param-hessian}
\nabla_{\boldsymbol\theta}^2\mathcal{L}
  = \mat{J}^\top \mat{A}_h\, \mat{J}
  \;+\; \sum_i \mat{r}_i\,\nabla_{\boldsymbol\theta}^2\widehat{T}_i .
\end{equation}
Near a good fit $\mat{r}\!\to\!0$ and the Gauss--Newton term $\mat{J}^\top\mat{A}_h \mat{J}$ governs the
curvature. Where $\mat{J}$ has full column rank,
$\kappa(\mat{J}^\top\mat{A}_h \mat{J})\le\kappa(\mat{J})^2\,\kappa_2(\mat{A}_h)$ (Appendix~\ref{app:continuous});
the strong-form Gauss--Newton Hessian $2\mat{A}_h^2$ gives the same bound with
$\kappa_2(\mat{A}_h)^2$ in place of $\kappa_2(\mat{A}_h)$. The Jacobian factor $\kappa(\mat{J})^2$ is
common to both, so \Cref{thm:energy-kappa}'s $\kappa^2\!\to\!\kappa$ reduction passes
through to the parameter-space conditioning the optimizer actually sees. The energy
rewrite adds no cost: the strong- and energy-form losses share the same matrix-free
stencil. 

\paragraph*{Matrix-free evaluation}
Evaluating either loss never requires $\mat{A}_h$ as a matrix, only its action
$\mat{A}_h\widehat{\mat{T}}$ on the predicted field, and that action is a local stencil on
the grid. Each control volume exchanges heat only with its six face-neighbors, so
the entry of $\mat{A}_h\widehat{\mat{T}}$ at a cell is the sum of the face fluxes
$G_f(\widehat{T}_j-\widehat{T}_i)$ to those neighbors plus its Robin boundary term,
the face conductances $G_f$ being precomputed once from the geometry. This action
is a handful of shifted-array differences and additions over the
$N\!\approx\!2\!\times\!10^6$-cell grid, evaluated at $O(N)$ cost and constant
memory; the $N\!\times\!N$ matrix is never built. The parameter gradient then comes
from a single reverse-mode pass through the same stencil, so automatic
differentiation sees only these array operations, not an assembled operator. We form
$\mat{A}_h$ explicitly only for the reference solver. Since the loss reads only the
predicted field, it applies to any architecture, not just the separable one used
here.

\subsection{Optimization with a Matrix-Valued Preconditioner}
\label{sec:muon2}

Training minimizes $\mathcal{L}(\boldsymbol\theta)$ over the network parameters by first-order
updates. The parameters that carry the remaining conditioning are matrix-shaped: each layer of
the branch network is a weight matrix $\mat{W}$, and at step $t$ the optimizer forms the
loss gradient $\mat{G}_t := \partial\mathcal{L}/\partial \mat{W}$ and its momentum $\mat{M}_t$, then steps
along a preconditioned version of $\mat{M}_t$. Adam preconditions \emph{entrywise},
rescaling each entry of $\mat{M}_t$ by its own running magnitude; but by
Eq.~\eqref{eq:param-hessian} the conditioning left after the energy rewrite (still
$\kappa_2(\mat{A}_h)\sim 10^4$ on our benchmark) is off-diagonal in $\mat{W}$: the stiff and
compliant directions mix many entries, and no entrywise rescaling can remove that.
A matrix-valued preconditioner must act jointly across the rows and columns of
$\mat{W}$, as orthogonalizing $\mat{M}_t$ does.

Muon~\cite{jordan2024muon} is such a preconditioner. Its key operation replaces a
matrix $\mat{M}$ with the matrix that has the same singular directions but all singular
values equal to one,
\begin{equation}
\label{eq:ns5}
\mathrm{orth}(\mat{M}) := \mat{M}\,(\mat{M}^\top \mat{M})^{-1/2},
\end{equation}
the orthogonal factor of the polar decomposition of $\mat{M}$. Applied to a gradient
momentum, $\mathrm{orth}$ makes the update act with equal strength along every
direction of $\mat{W}$ instead of being dominated by the few directions in which the
momentum is largest; this is the joint rescaling that a diagonal optimizer such
as Adam cannot perform. The inverse square root is never formed: a degree-$5$
Newton--Schulz (NS5) iteration approximates $\mathrm{orth}(\mat{M})$ with a few matrix
multiplications.

Muon already removes most of this conditioning, but its Newton--Schulz step has a
limitation that matters here. The polynomial that realizes
$\mathrm{orth}(\cdot)$ drives a singular value $\sigma$ toward $1$ only when $\sigma$ already lies
in its basin; for $\sigma$ near zero it is nearly flat and leaves the smallest singular
directions of the momentum almost unchanged. Under the conditioning of $\mat{A}_h$,
however, the informative low-curvature modes concentrate in just these directions,
so plain Muon under-corrects them.
Muon$^2$~\cite{liu2026muon} restores these directions by prepending an Adam-style
second-moment rescale $\mat{M}_t\oslash(\sqrt{\mat{V}_t}+\epsilon)$ that raises them to $O(1)$
magnitude before orthogonalization, so the polynomial can then drive them to unit
scale. Muon$^2$ keeps a momentum $\mat{M}_t$ and an elementwise second moment $\mat{V}_t$,
rescales, and orthogonalizes:
\begin{align}
\label{eq:muon2}
\mat{M}_t &= \mu\,\mat{M}_{t-1} + \mat{G}_t, \quad
\mat{V}_t = \beta\,\mat{V}_{t-1} + (1-\beta)\,\mat{G}_t \odot \mat{G}_t, \nonumber \\
\mat{W}_{t+1} &= \mat{W}_t - \eta\;\mathrm{orth}\!\bigl(\mat{M}_t \oslash (\sqrt{\mat{V}_t} + \epsilon)\,\bigr),
\end{align}
where $\odot$ and $\oslash$ are the elementwise (Hadamard) product and division,
the square root is elementwise, and $\mu,\beta,\eta,\epsilon$ are the momentum,
second-moment, step-size, and stabilization constants. As in Muon, the orthogonalized update is
further scaled by $0.2\,\sqrt{d_{\max}/d_{\min}}$, with $d_{\max}$ and $d_{\min}$
the larger and smaller dimensions of $\mat{W}$, so its root-mean-square size matches an
Adam step.

Because $\mathrm{orth}$ is defined for 2D matrices, we apply Muon$^2$ to the 2D MLP
branch weights and train the 3D ChebyKAN trunk tensors and all 1D parameters with
Adam; this is the configuration used throughout. The pretraining study
(\Cref{sec:pretrain-results}) bears out this hierarchy: Adam plateaus on the energy
loss, Muon removes most of the conditioning, and Muon$^2$ reaches the lowest error.

\section{Self-Improving Operator-Learning Framework}
\label{sec:pipeline}

\begin{figure}[!t]
  \centering
  \includegraphics[width=\columnwidth]{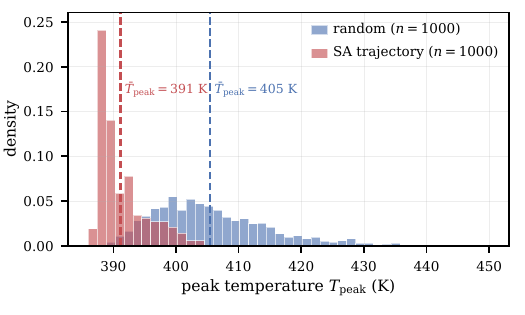}
  \caption{Distribution shift along the search. Peak temperatures of $1000$
  random placements (the pretraining distribution, mean $\approx 405$~K, broad)
  and of $1000$ placements visited along an optimization trajectory (mean
  $\approx 391$~K, concentrated in the low-temperature tail). The optimizer drives
  the design into a region the random pretraining distribution rarely samples, so
  the surrogate is evaluated far from where most of its training examples lie.}
  \label{fig:dist_shift}
\end{figure}

Before the search runs, the surrogate can only be pretrained on randomly generated
placements, since the placements the optimizer will favor are not yet known. The
search does not stay there: the optimizer drives the design toward low-peak
placements rare under the random pretraining distribution, shifting where the
surrogate is evaluated away from where it was trained.
\Cref{fig:dist_shift} shows the size of the shift: across a run, the peak temperatures of
the placements the search visits concentrate near $391$~K, well into the
low-temperature tail of the random pretraining distribution centered near
$405$~K. The surrogate is therefore queried where its pretraining examples are sparsest, and its accuracy on the broad pretraining set need not carry to this shifted region.

DeepOHeat-v2 closes this gap online. Along the search trajectory a trust gate
flags placements for the solver to refine, and the surrogate periodically retrains
on the refined solutions, so it improves where the search goes and steers the
search toward better placements (\Cref{fig:self_improve}). Refinement and retraining run only in a fixed number of early rounds; afterward
the search runs on the adapted surrogate alone. Each refinement thus corrects the
current step and, through retraining, improves later predictions.

\begin{figure*}[!t]
  \centering
  \includegraphics[width=0.85\textwidth]{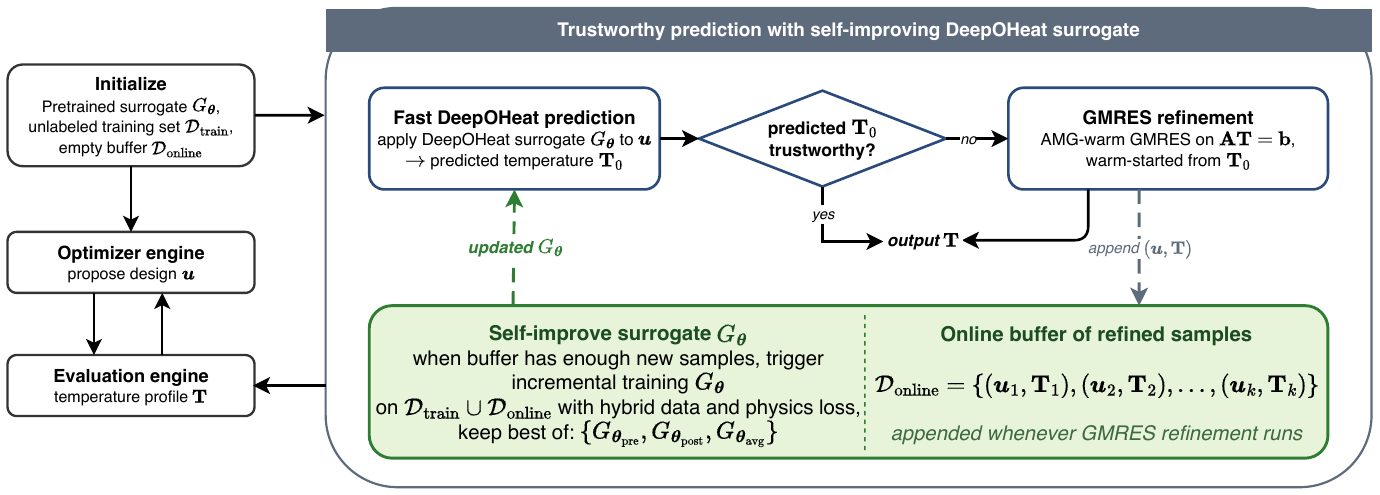}
  \caption{The DeepOHeat-v2 self-improving loop. The optimizer proposes a
  placement; the surrogate predicts its temperature field; the trust gate either
  accepts the prediction or sends it to a warm-started GMRES refinement. Refined
  solutions accumulate in an online buffer that periodically retrains the
  surrogate, so the surrogate grows more accurate along the trajectory the search
  follows.}
  \label{fig:self_improve}
\end{figure*}

\subsection{Online Adaptation: Incremental Training and Model Selection}
\label{sec:adapt}

\paragraph*{Incremental training}
We write $\mathcal{D}_{\mathrm{train}}$ for the fixed set of randomly generated
placements used in pretraining (\Cref{sec:pretrain}) and $\mathcal{D}_{\mathrm{online}}$
for the online buffer of solver-refined, in-trajectory placements accumulated
during the run (\Cref{sec:repair}); incremental training combines the two.
Each incremental-training event continues pretraining: the same architecture, the
same optimizer (Muon$^2$ on the weight matrices, Adam on the rest), and the same
optimizer state (momenta and second moments), carried across events from the
pretrained checkpoint. The loss keeps the energy term on both $\mathcal{D}_{\mathrm{train}}$ and
$\mathcal{D}_{\mathrm{online}}$ and adds a data mean-squared-error (MSE) term on the refined labels:
\begin{align}
\label{eq:inc-loss}
\mathcal{L}_{\mathrm{inc}} &= \lambda_e\bigl[\mathcal{L}_{\mathrm{energy}}(\mathcal{B}_{\rm orig})
                       + \mathcal{L}_{\mathrm{energy}}(\mathcal{B}_{\rm refined})\bigr] \nonumber \\
      &\quad + \lambda_d\,
        \mathrm{MSE}\bigl(G_{\boldsymbol{\theta}}(\mathcal{B}_{\rm refined}),
                          \mat{T}_{\rm refined}\bigr).
\end{align}
Two batch streams are sampled per step: $\mathcal{B}_{\rm orig}$
from $\mathcal{D}_{\mathrm{train}}$ and $\mathcal{B}_{\rm refined}$ from $\mathcal{D}_{\mathrm{online}}$. The energy stream
sees both datasets; the data-MSE stream sees $\mathcal{D}_{\mathrm{online}}$ only. We set
$\lambda_e$ and $\lambda_d$ so the two terms contribute comparable
magnitudes at the start of an event, and hold the incremental learning rates
constant and well below their pretraining values. Each event runs for a fixed
number of gradient steps. Keeping the energy term on $\mathcal{D}_{\mathrm{train}}$ active
throughout prevents the data-MSE stream from over-specializing the surrogate to the
narrow trajectory distribution.

\paragraph*{In-trajectory model selection}
An incremental event can overfit the buffer or step into a worse minimum, so we
keep its update only if it lowers error on held-out trajectory data. At the start of an event the buffer $\mathcal{D}_{\mathrm{online}}$ is split chronologically:
the most recent samples form a held-out validation slice $\mathcal{D}_{\mathrm{online}}^{\rm val}$ and the
older remainder is used for training.
After the event's gradient steps we form three candidates, the pre-update
model, the updated model, and their weight average~\cite{wortsman2022soup},
\begin{align}
G_{\boldsymbol{\theta}_{\mathrm{pre}}},\quad G_{\boldsymbol{\theta}_{\mathrm{post}}},\quad
G_{\boldsymbol{\theta}_{\mathrm{avg}}} = G_{(\boldsymbol{\theta}_{\rm pre} + \boldsymbol{\theta}_{\rm post})/2},
\end{align}
score each by the mean absolute peak temperature error on $\mathcal{D}_{\mathrm{online}}^{\rm val}$, and deploy the argmin.
Because $G_{\boldsymbol{\theta}_{\mathrm{pre}}}$ is a candidate, the deployed model is never worse than the
pre-event one on $\mathcal{D}_{\mathrm{online}}^{\rm val}$; when the selection keeps $G_{\boldsymbol{\theta}_{\mathrm{pre}}}$, the
optimizer state is also reverted to its pre-event snapshot so no rejected
gradient leaks forward. The averaged candidate $G_{\boldsymbol{\theta}_{\mathrm{avg}}}$ offers a midpoint that sometimes generalizes
better than either endpoint, at the cost of one more validation pass.

\paragraph*{A fixed number of adaptation events}
Refinement and incremental training run for a fixed number of events $C$. After the
$C$th event the surrogate has adapted to the trajectory distribution, so every later
proposal is decided on the surrogate alone, with no further solver calls or
retraining. The limit is needed because the trust gate (\Cref{sec:rhotk}) flags a fixed
fraction of proposals by design; without the limit, refinement would continue indefinitely
even after the surrogate is accurate.

\subsection{When to Call the Non-AI Solver: The Hotspot Trust Gate}
\label{sec:rhotk}

The surrogate is accurate on most placements; the gate must find the few with
large peak temperature error. Catching them takes a signal that \emph{ranks}
predictions by peak error, which the v1 global residual cannot do, and a threshold
that follows the residual scale as it drifts with the search and with each
retraining. The hotspot-localized residual $r_{\mathrm{hot}}$ supplies the ranking; a
sliding-window percentile sets the threshold.

\paragraph*{Why the global residual may fail to rank}
DeepOHeat-v1~\cite{yu2025deepoheat} gates on the global relative residual
$r_{\mathrm{rel}} = \|\mat{r}\|_2/\|\mat{b}_h\|_2$ (\Cref{sec:v1}), with residual
$\mat{r} := \mat{A}_h\widehat{\mat{T}} - \mat{b}_h$, accepting when $r_{\mathrm{rel}} < \alpha$. For
$r_{\mathrm{rel}}$ to rank predictions it would have to track the peak error, but the two
are linked only through the global operator norm: with
$\mat{e} := \widehat{\mat{T}} - \mat{T}$ and $x^*$ the predicted hotspot,
\begin{align}
\label{eq:rrel-bound}
|e(x^*)| \le \|\mat{e}\|_2 \le \|\mat{A}_h^{-1}\|_2\,\|\mat{r}\|_2
        = \frac{\kappa_2(\mat{A}_h)}{\|\mat{A}_h\|_2}\,\|\mat{b}_h\|_2\cdot r_{\mathrm{rel}} .
\end{align}
The norm $\|\mat{A}_h^{-1}\|_2$ carries the full conditioning
$\kappa_2(\mat{A}_h)=6.02\!\times\!10^4$, so the bound evaluates to
$\approx 1.6\!\times\!10^5$~K, far above the sub-Kelvin errors we observe, and
leaves $r_{\mathrm{rel}}$ free to vary independently of the peak error. In practice it
does: on a held-out set of random placements (\Cref{sec:confidence-experiments})
the Spearman rank correlation between $r_{\mathrm{rel}}$ and the peak error is $\rho = -0.018$ ($\rho = \pm 1$ for a perfectly monotone ranking), so $r_{\mathrm{rel}}$ orders predictions essentially at random and the v1 gate can no longer separate reliable from unreliable predictions.

\paragraph*{The hotspot-localized residual}
The gate does not need the whole-field error; it needs the error at the hotspot.
We therefore measure the residual only there.
\begin{definition}[Hotspot-localized residual]
\label{def:rhotk}
For a prediction $\widehat{\mat{T}}$ with residual $\mat{r} := \mat{A}_h\widehat{\mat{T}} - \mat{b}_h$, let
$S_m$ index the top-$m$ hottest cells of $\widehat{\mat{T}}$, for a fixed count $m$
(\Cref{tab:hyper}). Define
\begin{align}
r_{\mathrm{hot}}(\widehat{\mat{T}}) := \frac{1}{m}\sum_{y \in S_m} |\mat{r}(y)|,
\end{align}
the mean residual magnitude over the predicted hotspot.
\end{definition}
Computing $r_{\mathrm{hot}}$ costs one sparse matrix--vector product, against the same
$\mat{b}_h$ the surrogate is trained on. Restricting the residual to the hotspot
gives a matching a-posteriori bound (Appendix~\ref{app:bound}) in which the global
norm $\|\mat{A}_h^{-1}\|_2$ of \eqref{eq:rrel-bound} is replaced by the local
diagonal entry $(\mat{A}_h^{-1})_{x^*,x^*}$, which does \emph{not} grow with
$\kappa_2(\mat{A}_h)$. Like the v1 bound it is numerically loose, so $r_{\mathrm{hot}}$ ranks
predictions rather than certifying any one of them. Averaging over the top-$m$
cells, instead of the single hottest, keeps the ranking stable when the predicted
and true hotspots are close but not identical.

\paragraph*{Adaptive sliding-window flagging}
The scale of $r_{\mathrm{hot}}$ is not stationary: it shrinks as the search narrows and shifts again each time the surrogate is updated, so a fixed threshold calibrated once on $\mathcal{D}_{\mathrm{train}}$ would soon flag too many or too few proposals. We instead set the threshold relative to the residuals currently
observed, maintaining a first-in-first-out (FIFO) window $\mathcal{W}$ of the most recent
$r_{\mathrm{hot}}$ values and, while the solver is active, taking the threshold at search step
$t$ to be the $(1 - f_{\mathrm{target}})$-percentile of $\mathcal{W}$. This makes the threshold
self-calibrating: it tracks the moving $r_{\mathrm{hot}}$ distribution and
holds the flag rate near $f_{\mathrm{target}}$ regardless of search progress or model
updates. The window is bootstrapped by an initial block of no-flag search steps so
the threshold is well-defined once flagging begins.

\subsection{How to Call the Non-AI Solver: Warm-Started GMRES Refinement}
\label{sec:repair}

When a proposal is flagged, we want a peak estimate more accurate than the
surrogate but far cheaper than a cold solve. We run the inherited GMRES
refinement (\Cref{sec:v1}) on $\mat{A}_h\mat{T} = \mat{b}_h$ from the warm start
$\mat{T}_0 = \widehat{\mat{T}}$, with a Ruge--St\"uben algebraic-multigrid (AMG) V-cycle
preconditioner~\cite{ruge1987amg,stuben2001review}.
Because the warm prediction is already within a few K of the true field, the
refinement need not drive the residual far down, so we set a deliberately loose
relative tolerance $\epsilon$. In practice this tolerance is met after a single
V-cycle, which propagates information across all spatial scales and roughly
halves the warm-start residual. The cost is $\sim\!1$~s per
refinement versus $\sim\!16$~s for a tight solve (a $\sim\!15\!\times$ saving),
and the post-cycle peak temperature error is typically reduced by $1$--$2$~K,
enough to flip a borderline accept/reject decision before it is committed. Each
refined placement and its solved field $(\mat{u}', \mat{T})$ are then added to the
online buffer $\mathcal{D}_{\mathrm{online}}$, the labeled in-trajectory data on which the
surrogate is incrementally retrained (\Cref{sec:adapt}).

\subsection{Algorithm}
\label{sec:algorithm}

\Cref{alg:v2} consolidates the components above into the decision logic of one
optimization run; \Cref{fig:self_improve} gives the complementary data-and-model
view. The run proceeds in two phases: while $e < C$, every proposal passes through
the gate, and a flagged proposal is refined, buffered, and, once the buffer crosses
the next event threshold, drives an incremental-training event with model
selection; once $e$ reaches $C$, this block is skipped and the search continues on
the adapted surrogate alone. The peak estimate $\widehat{p}$ in the accept/reject
test holds the surrogate's predicted peak by default and the refined peak whenever
the proposal was refined. A single tight FVM solve at the end verifies the
best-so-far placement and returns the peak temperature the run reports.

\begin{algorithm}[!t]
\caption{DeepOHeat-v2 (one optimization run): decision logic.}
\label{alg:v2}
\begin{algorithmic}[1]
\Require Pretrained $G_{\boldsymbol{\theta}}$; train set $\mathcal{D}_{\mathrm{train}}$; acceptance schedule
$(\tau_0, \tau_{\min}, N)$; flag rate $f_{\mathrm{target}}$; window $|\mathcal{W}|$;
refine tolerance $\epsilon$; event trigger $\Delta N$; event budget $C$.
\State $\mathcal{D}_{\mathrm{online}} \!\gets\! \emptyset$; $\mat{u}^0 \!\sim\! \mathcal{U}_{\rm init}$;
       $\mathcal{W} \!\gets\! \emptyset$; event count $e \!\gets\! 0$.
\State Bootstrap $\mathcal{W}$ with $|\mathcal{W}|$ no-flag search steps.
\For{$t = 1, \ldots, N$}
  \State $\mat{u}' \gets \text{perturb}(\mat{u}^{t-1})$;\;
         $\widehat{\mat{T}} \gets G_{\boldsymbol{\theta}}(\mat{u}')$;\;
         $\widehat{p} \gets \|\widehat{\mat{T}}\|_\infty$.
  \If{$e < C$}\Comment{solver still active}
    \State Push $r_{\mathrm{hot}}(\widehat{\mat{T}})$ into $\mathcal{W}$.
    \If{$r_{\mathrm{hot}}(\widehat{\mat{T}}) > Q_{1-f_{\mathrm{target}}}(\mathcal{W})$}\Comment{flag at window percentile}
      \State $\mat{T} \gets \text{AMG-GMRES}(\mat{A}_h, \mat{b}_h, \mat{T}_0{=}\widehat{\mat{T}}, \epsilon)$.
      \State $\widehat{p} \gets \|\mat{T}\|_\infty$;\;
             $\mathcal{D}_{\mathrm{online}} \gets \mathcal{D}_{\mathrm{online}} \cup \{(\mat{u}', \mat{T})\}$.
      \If{$|\mathcal{D}_{\mathrm{online}}| \ge (e{+}1)\,\Delta N$}\Comment{event}
        \State $\boldsymbol{\theta}_{\rm pre} \gets \boldsymbol{\theta}$;\;
               incrementally train $\boldsymbol{\theta}_{\rm post}$ on $\mathcal{L}_{\mathrm{inc}}$.
        \State $\boldsymbol{\theta}_{\rm avg} \gets
               \tfrac{1}{2}(\boldsymbol{\theta}_{\rm pre} {+} \boldsymbol{\theta}_{\rm post})$.
        \State $\boldsymbol{\theta} \gets \arg\min_{s} \,
               \overline{|\Delta T_{\rm peak}|}(G_{\boldsymbol{\theta}_s}, \mathcal{D}_{\mathrm{online}}^{\rm val})$,
               $s\in\{\rm pre,post,avg\}$.
        \State $e \gets e + 1$.
      \EndIf
    \EndIf
  \EndIf
  \State Accept $\mat{u}'$ w.p.\
         $\min\bigl(1, \exp(-(\widehat{p}{-}\widehat{p}^{t-1})/\tau_t)\bigr)$.
\EndFor
\State Final tight FVM verify on the best-so-far placement.
\end{algorithmic}
\end{algorithm}

\section{Experiments}
\label{sec:experiments}

\subsection{The F2F Chiplet Benchmark}
\label{sec:benchmark-stack}

The benchmark is a $3\!\times\!3$~mm face-to-face (F2F) chiplet stack
(\Cref{fig:chip_model}): two bulk‑plus‑active‑silicon dies vertically stacked and bonded face‑to‑face through a hybrid Cu–Cu bond, mounted on an organic substrate. Cooling is
asymmetric, with a water cold plate at the top ($h = 5000$~W/m$^2$K) and weak
board conduction at the bottom ($h = 50$~W/m$^2$K). Copper TSVs, modeled as
$10\!\times\!10~\mu$m columns on a $50~\mu$m pitch, pass through the silicon layers. \Cref{tab:layers} lists the seven layers and their
conductivities, which span $0.5$ to $400$~W/m$\cdot$K, an $800\times$ contrast
between copper and the substrate. We discretize the stack on a uniform
$300\!\times\!300\!\times\!23$ grid ($10~\mu$m in plane;
$N \approx 2.07\!\times\!10^6$ cells), yielding the operator $\mat{A}_h$
of Eq.~\eqref{eq:fvm}.

Ten power blocks sit on the two dies and dissipate $\approx\!3.5$~W in total: four
hotspots (PHY, PLL, and two CPU cores) and six lower-power background blocks, all
listed in \Cref{tab:blocks}. The design places these blocks, so the
configuration $\mat{u}$ collects their $(x,y)$ positions; blocks on the same die
may not overlap, though cross-die overlap is allowed.

The high conductivity contrast makes the benchmark challenging for DeepOHeat-v1. It leaves the
discretized operator $\mat{A}_h$ ill-conditioned: a Lanczos estimate gives
$\kappa_2(\mat{A}_h) = 6.02\!\times\!10^4$ (computed in Appendix~\ref{app:bound}), far
above the regime in which DeepOHeat-v1 was demonstrated.

\begin{figure}[!t]
  \centering
  \includegraphics[width=\columnwidth]{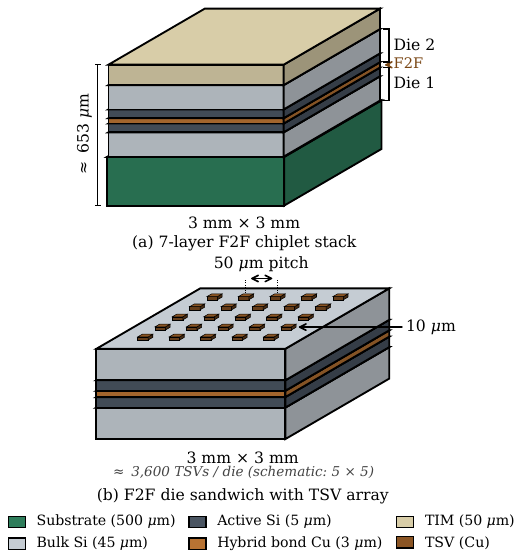}
  \caption{The F2F chiplet benchmark: two bulk-plus-active-Si dies vertically stacked and bonded face-to-face on an organic substrate, cooled by a top cold plate and weak bottom board conduction, with copper TSVs passing through the silicon layers.}
  \label{fig:chip_model}
\end{figure}

\begin{table}[!t]
\centering
\caption{Layer stack of the F2F chiplet benchmark (bottom $\to$ top); TIM is
the thermal interface material.}
\label{tab:layers}
\begin{tabular}{cl r r}
\toprule
\# & Layer & Thickness & $k$ (W/m$\cdot$K) \\
\midrule
1 & Substrate (BT/FR4)        & $500~\mu$m & $0.5$ \\
2 & Die 1 bulk Si (+ TSV Cu)  & $45~\mu$m  & $140 / 400$ \\
3 & Die 1 active Si (+ TSV)   & $5~\mu$m   & $140 / 400$ \\
4 & Hybrid bond (Cu--Cu)      & $3~\mu$m   & $100$ \\
5 & Die 2 active Si (+ TSV)   & $5~\mu$m   & $140 / 400$ \\
6 & Die 2 bulk Si (+ TSV Cu)  & $45~\mu$m  & $140 / 400$ \\
7 & TIM                       & $50~\mu$m  & $5$ \\
\bottomrule
\end{tabular}
\end{table}

\begin{table}[!t]
\centering
\caption{Power blocks on the two dies ($\approx\!3.5$~W total). PHY, PLL, and the two
CPU cores are hotspots; the rest are background blocks.}
\label{tab:blocks}
\begin{tabular}{llr}
\toprule
Die & Block & Power (W) \\
\midrule
1 & PHY    & $1.0$  \\
1 & PLL    & $0.6$  \\
1 & SRAM0  & $0.08$ \\
1 & SRAM1  & $0.08$ \\
1 & Analog & $0.08$ \\
2 & CPU0   & $0.7$  \\
2 & CPU1   & $0.7$  \\
2 & GPU    & $0.08$ \\
2 & Cache  & $0.08$ \\
2 & IO     & $0.08$ \\
\bottomrule
\end{tabular}
\end{table}

\subsection{Experimental Setup}
\label{sec:setup}

All experiments use the F2F benchmark of \Cref{sec:benchmark-stack} under a single
fixed protocol. We check the two contributions on a held-out set of $n=100$ random placements, a
loss/optimizer study for the data-free training recipe of \Cref{sec:pretrain} and
a catch-rate test for the $r_{\mathrm{hot}}$ trust gate of \Cref{sec:rhotk}, then put the
full loop through a simulated-annealing (SA) run from a random initial placement.

\paragraph*{Simulated-annealing schedule}
All optimization runs use the same SA loop~\cite{vanlaarhoven1987sa}: from a random
non-overlapping placement, each of $N = 1000$ iterations perturbs one block by
at most $\Delta_{\max} = 10$ grid steps and accepts the move with the
Metropolis criterion at an acceptance temperature $\tau_t$ that is
geometrically cooled,
\begin{align}
\label{eq:cooling}
\tau_t = \tau_0\,(\tau_{\min}/\tau_0)^{t/N},
\qquad \tau_0 = 5,\ \tau_{\min} = 0.01 .
\end{align}
The cost is $T_{\mathrm{peak}}(\mat{u})$, so each
iteration requires one thermal evaluation.

\paragraph*{Implementation}
The separable operator network (\Cref{sec:v1}) uses an $8$-layer branch of width
$256$, degree-$3$ ChebyKAN trunks of width $64$, and rank $r = 128$. It is
implemented in JAX/Equinox~\cite{equinox2022}; the FVM assembly, the tight oracle
solves, and the AMG-warm GMRES refinement use Ruge--St\"uben
algebraic-multigrid preconditioning via PyAMG~\cite{bell2022pyamg}. All
experiments run on a single NVIDIA RTX~3090.

\paragraph*{Self-improving-loop hyperparameters}
The trust-gate, refinement, and incremental-training hyperparameters of
\Cref{sec:pipeline} are collected in \Cref{tab:hyper}; they were set once and
held fixed across all runs.

\begin{table}[!t]
\centering
\caption{DeepOHeat-v2 self-improving-loop hyperparameters, fixed across all runs.}
\label{tab:hyper}
\small
\setlength{\tabcolsep}{6pt}
\begin{tabular}{lr}
\toprule
Quantity & Value \\
\midrule
Trust-window length ($|\mathcal{W}|$)                         & $100$ \\
Hotspot cells ($m$)                                & $100$ \\
Target flag rate ($f_{\mathrm{target}}$)           & $0.30$ \\
Refinement rel.\ tolerance ($\epsilon$)            & $5\!\times\!10^{-1}$ \\
\addlinespace
Loss weights ($\lambda_e,\ \lambda_d$)             & $0.1,\ 1$ \\
Incremental LRs ($\eta_{\rm muon},\ \eta_{\rm adam}$) & $5\!\times\!10^{-6},\ 1\!\times\!10^{-6}$ \\
Batch size (per stream)                            & $4$ \\
Gradient steps per event                           & $1000$ \\
\addlinespace
Refinements per event ($\Delta N$)                 & $50$ \\
Event budget ($C$)                                 & $3$ \\
Validation-slice size                              & $10$ \\
\bottomrule
\end{tabular}
\end{table}

\paragraph*{Evaluated configurations and metrics}
We compare four configurations that share the SA schedule above and
just-in-time (JIT)-compiled inference, differing only in the per-proposal cost policy.
\emph{Surrogate-only} drives SA with $G_{\boldsymbol{\theta}}$'s predicted peak and issues no solves during the search, with a single tight FVM verify at the end for the true peak; without the trust gate, it never flags or refines a proposal.
\emph{DeepOHeat-v1} adds flagging and AMG-warm GMRES refinement but never updates
$G_{\boldsymbol{\theta}}$; we deliberately give it the $r_{\mathrm{hot}}$ gate rather than its own
$r_{\mathrm{rel}}$, so anything v2 gains over it is the online adaptation. \emph{DeepOHeat-v2}
adds that adaptation (\Cref{sec:adapt}). \emph{Oracle SA} drives the identical
loop with a tight FVM solve at every step (tolerance $10^{-7}$,
AMG-preconditioned GMRES, $\sim\!16$~s per evaluation); it is a non-competing
ground-truth upper bound.
For each configuration we report the surrogate-predicted peak (the signal the
optimizer acts on), the true (oracle) peak of the returned placement under a
tight FVM solve, their difference (the surrogate--true gap, i.e.\ the prediction
error at the returned placement), and wall time.

\subsection{Pretraining: Loss Form and Optimizer}
\label{sec:pretrain-results}

Both design choices behind the data-free training recipe hold up here: the
energy-form loss and the Muon$^2$ optimizer. We pretrain the same separable surrogate on
$n_{\mathrm{data}} = 10{,}000$ unlabeled power maps (batch $32$, no temperature
labels) under either the strong-form residual or the energy-form loss, with
Adam, Muon, or Muon$^2$; a supervised control uses $n=100$ labeled pairs and a
pure MSE loss. Every loss-form and optimizer combination is in \Cref{tab:pretrain}, evaluated on
the $n=100$ held-out set after $40{,}000$ epochs, training times included. The
largest single change in the table is the loss form: under Adam, switching the
strong-form residual to the energy form cuts peak temperature error from
$30.88$~K to $1.05$~K, a $97\%$ drop. Both losses share the same FVM
discretization, so interface representation is identical and only the
conditioning differs, consistent with the $\kappa^2\!\to\!\kappa$ reduction of
\Cref{thm:energy-kappa}. The supervised MSE control, trained on $n=100$ labeled fields, reaches only
$9.81$~K, far short of the energy form's $1.05$~K under the same optimizer. Its labels are not free: each of
the $100$ fields is a tight solve, and closing the gap would need many more. The
data-free energy form reaches sub-Kelvin accuracy with no labels at all.

Within the energy form, the optimizer adds a further gain: Muon cuts the error
$39\%$ below Adam, and Muon$^2$ another $14\%$ ($1.05 \to 0.64 \to 0.55$~K). This matches the mechanism of \Cref{sec:muon2}: Muon supplies the joint rescaling Adam cannot, and Muon$^2$ additionally restores the smallest singular directions; hence Muon$^2$ is our default.

The strong-form rows show the failure the energy form avoids. Under that loss Adam plateaus
about $30$~K from the optimum and both Muon variants diverge. The divergence
comes from the NS5 orthogonalization, which loses its contraction at this
conditioning ($\kappa^2 \approx 3.6\!\times\!10^9$); the lower $\kappa$ of the energy
form avoids it.

\Cref{fig:predictions} shows the energy$+$Muon$^2$ surrogate on held-out
placements that span a range of peak temperatures: it reproduces each reference
field closely, locating the hotspot, with the largest residuals confined to block
edges. At $0.55$~K mean peak temperature error, this is the pretrained $G_{\boldsymbol{\theta}}$ used in
the rest of the experiments.

\begin{table*}[!t]
\centering
\caption{Pretraining results on the $n=100$ held-out set after $40{,}000$
epochs: mean peak temperature error, relative $L_2$ field error, field
mean-absolute error (MAE), and wall-clock training time. $\dagger$~training diverged; values are the error magnitude at termination.}
\label{tab:pretrain}
\small
\setlength{\tabcolsep}{6pt}
\begin{tabular}{llrrrr}
\toprule
Loss form & Optimizer & Peak temp.\ error (K) & Relative $L_2$ & MAE (K) & Train (min) \\
\midrule
Supervised MSE      & Adam      & $9.81$  & $1.28\!\times\!10^{-2}$ & $3.87$ & $20.0$ \\
Strong-form physics & Adam      & $30.88$ & $2.91\!\times\!10^{-2}$ & $8.76$ & $9.5$ \\
Strong-form physics & Muon      & $\sim\!2.2\!\times\!10^{11}\,\dagger$ & $5.84\!\times\!10^{8}$ & $\sim\!2.2\!\times\!10^{11}$ & $11.8$ \\
Strong-form physics & Muon$^2$  & $\sim\!3.6\!\times\!10^{10}\,\dagger$ & $1.18\!\times\!10^{8}$ & $\sim\!4.4\!\times\!10^{10}$ & $11.9$ \\
\midrule
Energy-form physics & Adam      & $1.05$  & $6.79\!\times\!10^{-4}$ & $0.179$ & $11.7$ \\
Energy-form physics & Muon      & $0.639$ & $4.18\!\times\!10^{-4}$ & $0.112$ & $13.3$ \\
\textbf{Energy-form physics} & \textbf{Muon$^2$} & $\mathbf{0.551}$ & $\mathbf{3.89\!\times\!10^{-4}}$ & $\mathbf{0.103}$ & $\mathbf{13.5}$ \\
\bottomrule
\end{tabular}
\end{table*}

\begin{figure*}[!t]
  \centering
  \includegraphics[width=0.9\textwidth]{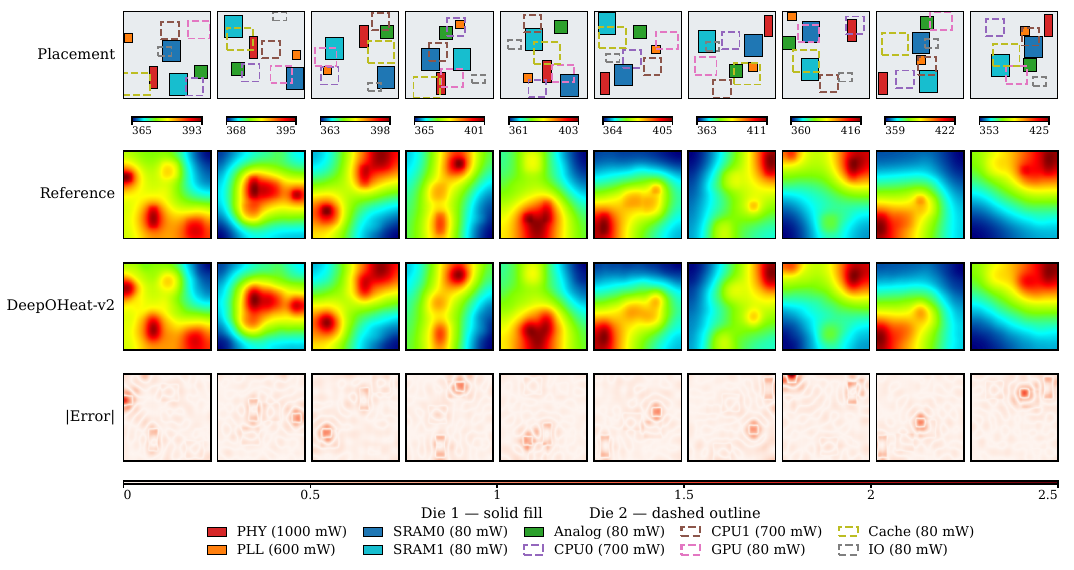}
  \caption{Energy$+$Muon$^2$ surrogate predictions on $10$ random held-out
  placements. Rows, top to bottom: block placement (Die~1 solid fill, Die~2
  dashed overlay) with reference and predicted peak temperatures; reference
  field $T_{\rm ref}$ at the Die-1 active layer; surrogate prediction $\widehat{\mat{T}}$
  at the same layer (shared color scale); absolute residual
  $|\widehat{\mat{T}} - T_{\rm ref}|$. The hotspot under PHY is captured consistently;
  residuals concentrate at block edges where in-plane gradients are steepest.}
  \label{fig:predictions}
\end{figure*}

\subsection{Confidence-Gate Validation}
\label{sec:confidence-experiments}

\begin{figure}[!t]
  \centering
  \includegraphics[width=0.99\columnwidth]{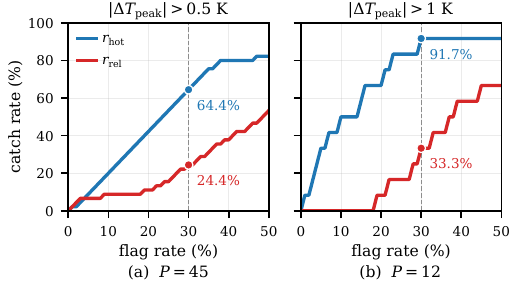}
  \caption{Catch-rate operating curves for $r_{\mathrm{hot}}$ vs.\ the global
  relative residual $r_{\mathrm{rel}}$ on the surrogate pretrained
  for $40{,}000$ epochs ($n = 100$ held-out placements). Each curve plots the fraction of bad
  predictions caught against the flag-rate budget; the diagonal is the random
  baseline. (a)~Bad $= |\Delta T_{\rm peak}|>0.5$~K ($P = 45$ positives).
  (b)~Bad $= |\Delta T_{\rm peak}|>1$~K ($P = 12$ positives). $r_{\mathrm{hot}}$
  dominates $r_{\mathrm{rel}}$ across the budget range.}
  \label{fig:confidence_roc}
\end{figure}

A surrogate that is accurate on average can still be wrong on individual
placements, so we test whether the $r_{\mathrm{hot}}$ gate flags those. On the
$n=100$ held-out set, $r_{\mathrm{hot}}$ ranks predictions by true error far better than the
global residual it replaces: its Spearman correlation with
peak temperature error is $\rho = +0.642$, against $\rho = -0.018$
for $r_{\mathrm{rel}}$. At a $30\%$ flag budget (\Cref{fig:confidence_roc}), $r_{\mathrm{hot}}$ catches
$64.4\%$ of $>\!0.5$~K errors and $91.7\%$ of $>\!1$~K errors, where $r_{\mathrm{rel}}$ stays
near the $30\%$ random baseline ($24.4\%$ and $33.3\%$).

The tail the gate catches is real, not an averaging artifact: $12\%$ of random
held-out placements have peak temperature error $>\!1$~K and $2\%$ exceed $2$~K,
and which placements they are cannot be known without a solve.

\subsection{Online Stability of the Adaptation Loop}
\label{sec:stability}

Two sanity checks before the integrated run: the adaptation loop never deploys an
update that worsens held-out error, and it does not forget the broad pretraining
distribution.

\paragraph*{Model selection}
At each event we compare the held-out
validation peak temperature error of $G_{\boldsymbol{\theta}_{\mathrm{pre}}}$, $G_{\boldsymbol{\theta}_{\mathrm{post}}}$, and
$G_{\boldsymbol{\theta}_{\mathrm{avg}}}$ and deploy the minimizer (\Cref{fig:stability}, left). At event~1 the
continued steps overshoot
on the still-broad early-SA validation slice, so $G_{\boldsymbol{\theta}_{\mathrm{pre}}}$ wins ($0.601$~K versus
$G_{\boldsymbol{\theta}_{\mathrm{post}}}$'s $0.700$~K) and the update is reverted; events~2 and~3 keep $G_{\boldsymbol{\theta}_{\mathrm{post}}}$ as
the validation slice concentrates on late-SA placements and the per-event error
scale drops from $\sim\!1$~K to $\sim\!0.2$~K. Without model selection the regressing
first event would have been committed.

\paragraph*{Anti-forgetting}
Held-out error on the broad random set drops slightly after adaptation, from the
pretrained $0.551$~K to $0.499$~K, so adapting to the trajectory does not cost
accuracy elsewhere. The
$\lambda_e\mathcal{L}_{\mathrm{energy}}(\mathcal{D}_{\mathrm{train}})$ term keeps the surrogate accurate on the broad
distribution while the data-MSE term fits the trajectory; the two do not conflict.

\begin{figure}[!t]
  \centering
  \includegraphics[width=0.99\columnwidth]{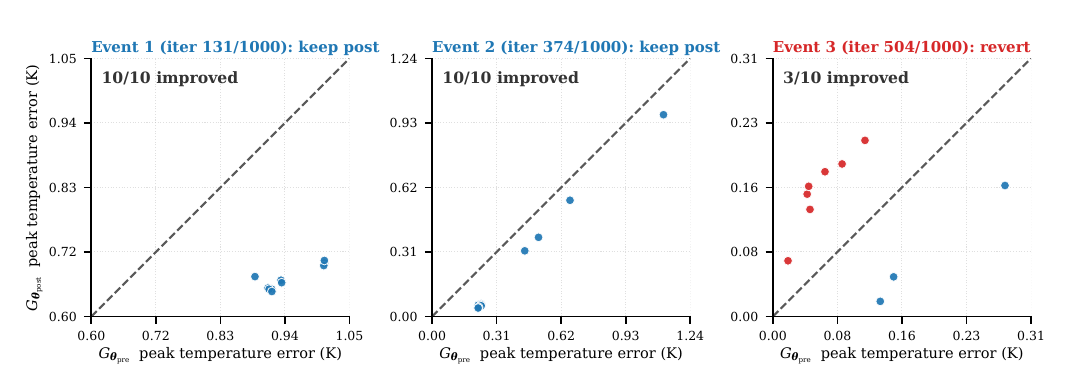}
  \caption{Online stability of DeepOHeat-v2 on the integrated run. \emph{Left:}
  model-selection trace; for each incremental-training event the $10$
  most-recent validation samples are plotted at $(G_{\boldsymbol{\theta}_{\mathrm{pre}}}$ error, $G_{\boldsymbol{\theta}_{\mathrm{post}}}$
  error$)$, with points below the diagonal improved and above regressed.
  Event~1 reverts to $G_{\boldsymbol{\theta}_{\mathrm{pre}}}$ (red), events~2--3 keep $G_{\boldsymbol{\theta}_{\mathrm{post}}}$ (blue); the
  per-event scale drops from $\sim\!1$~K to $\sim\!0.2$~K as SA cools.
  \emph{Right:} anti-forgetting; the post-pipeline held-out mean peak
  temperature error stays below the pretrained baseline, so the loop does not
  forget the broad distribution.}
  \label{fig:stability}
\end{figure}

\subsection{Integrated Run: Optimization Quality and Cost}
\label{sec:integrated}

From a random initial placement, the full loop reaches a design whose true peak is
within $0.11$~K of what its own surrogate predicted, in $292$~s, about $56\times$
faster than solving at every step and with a bounded number of FVM solves.

\paragraph*{The surrogate--true gap}
The surrogate--true gap shrinks across the three configurations, from $1.12$~K
for surrogate-only, to $0.78$~K for DeepOHeat-v1, to $0.11$~K for DeepOHeat-v2
(\Cref{tab:case-study}). The gap is prediction error the optimizer never sees during the search, not a
quality ranking of the placements. Under-prediction is the dangerous direction, since it lets the
search accept a placement that is hotter than it looks; it dominates the
surrogate-only and v1 gaps, while v2's gap is a small over-prediction. Each gap has a mechanism: surrogate-only never sees its own error; v1 catches and
refines bad proposals but leaves the surrogate fixed; only v2 adapts it, and that
closes the gap to oracle level.

\paragraph*{Per-configuration results}
The run starts at a peak of $397.78$~K; \Cref{tab:case-study} and \Cref{tab:cost}
report the per-configuration outcomes.
With no gate to warn it, surrogate-only commits its $1.12$~K error undetected;
v1 and v2 gate every proposal and refine the flagged ones. Oracle SA appears
only as the reference row: solving tightly at every step, it has no
surrogate--true gap and optimizes on ground truth throughout.

\begin{table}[!t]
\centering
\caption{Predicted vs.\ true peak on the integrated run from a random initial
placement. The gap is the surrogate--true difference at the returned placement;
Oracle SA is the non-competing zero-gap reference.}
\label{tab:case-study}
\small
\setlength{\tabcolsep}{6pt}
\begin{tabular}{lrrr}
\toprule
Configuration & \shortstack[r]{Predicted\\peak (K)} & \shortstack[r]{True/oracle\\peak (K)} & \shortstack[r]{Gap\\(K)} \\
\midrule
Surrogate-only        & $387.67$          & $388.79$          & $1.12$          \\
DeepOHeat-v1          & $387.25$          & $388.03$          & $0.78$          \\
\textbf{DeepOHeat-v2} & $\mathbf{387.15}$ & $\mathbf{387.04}$ & $\mathbf{0.11}$ \\
\midrule
Oracle SA (ref)       & ---               & $387.36$          & $0.00$          \\
\bottomrule
\end{tabular}
\end{table}

\paragraph*{Cost}
DeepOHeat-v2 reaches the oracle's placement quality for far less compute: $292$~s
against Oracle SA's $4.51$~h, with all four configurations within $\sim\!1.8$~K of
one another on the absolute true peak (\Cref{tab:cost}). The $0.32$~K by which v2's true peak ($387.04$~K)
sits below Oracle SA's ($387.36$~K) is not surrogate superiority: SA is a
stochastic search whose accept/reject sequence is perturbed by prediction noise,
so the two trajectories settle in slightly different basins.

\begin{table}[!t]
\centering
\caption{True (oracle) peak, FVM solves issued during SA, and wall time on the
integrated run; ``FVM solves in SA'' excludes the single final verification
solve.}
\label{tab:cost}
\small
\setlength{\tabcolsep}{4pt}
\begin{tabular}{lrrr}
\toprule
Configuration & \shortstack[r]{True peak\\(K)} & \shortstack[r]{FVM solves\\in SA} & Wall \\
\midrule
Surrogate-only        & $388.79$          & $0$            & $40$~s \\
DeepOHeat-v1          & $388.03$          & $\sim\!300$    & $335$~s \\
\textbf{DeepOHeat-v2} & $\mathbf{387.04}$ & $\mathbf{150}$ & $\mathbf{292}$~s \\
\midrule
Oracle SA (ref)       & $387.36$          & $1000$         & $4.51$~h \\
\bottomrule
\end{tabular}
\end{table}

\paragraph*{Optimized placements}
\Cref{fig:optimized_placements} shows the best placement each configuration returns
alongside its oracle FVM field at the Die-1 active layer: surrogate-only leaves PHY near the chip edge, producing
a hotspot it under-predicts; v1's refinement moves the placement inward; and v2's
adapted surrogate finds a PHY position whose oracle field is visually
indistinguishable from Oracle SA's; the $0.32$~K between their peaks is basin
scatter, not surrogate error.

\begin{figure*}[!t]
  \centering
  \includegraphics[width=0.85\textwidth]{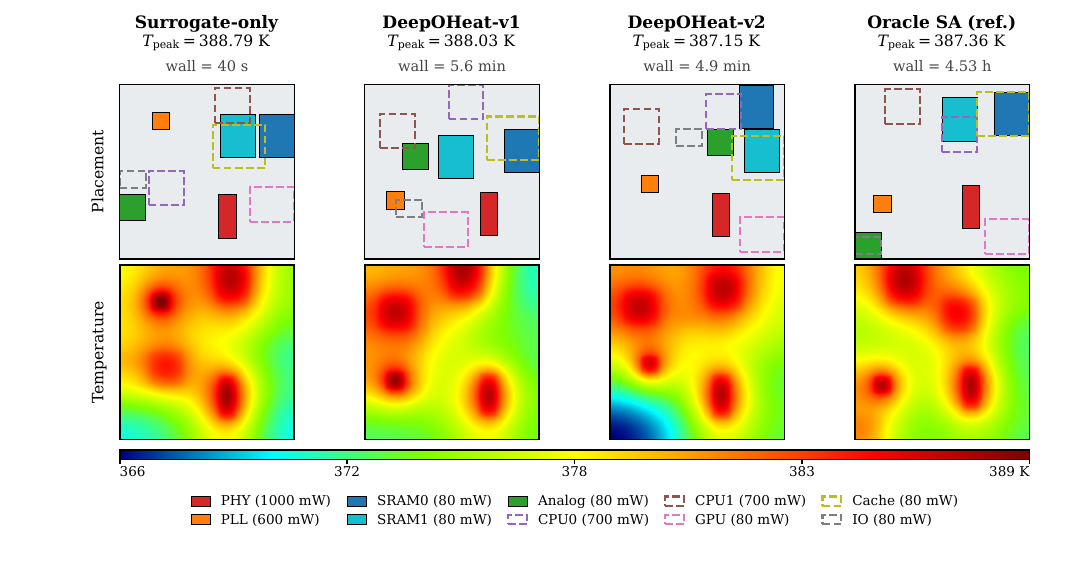}
  \caption{Optimized placements and oracle FVM temperature fields at the Die-1
  active layer for the four configurations on the integrated run. Row~1:
  placement with Die~1 solid outlines and Die~2 dashed overlays. Row~2: oracle
  FVM field, shared color scale. DeepOHeat-v2 matches Oracle SA in placement
  structure and, to within the SA basin scatter ($\sim\!0.3$~K), in peak temperature.}
  \label{fig:optimized_placements}
\end{figure*}

\section{Conclusion}
\label{sec:conclusion}

We have presented DeepOHeat-v2, a self-improving operator-learning framework for
thermal-aware placement optimization in 3D-IC design. A data-free thermal surrogate must clear two hurdles before it can steer an
optimizer: it must first be trainable without labels on
high-contrast multi-die stacks, and it must then stay accurate as the search
drives placements away from the distribution it was trained on. To make
data-free training converge, we cast the physics loss in energy form, which
provably reduces the prediction-space loss-Hessian conditioning from
$\kappa_2(\mat{A}_h)^2$ to $\kappa_2(\mat{A}_h)$, and pair it with a matrix-preconditioned
optimizer that succeeds where first-order physics-informed training otherwise
stalls. Holding that accuracy under the distribution shift of
optimization falls to the self-improving loop: the hotspot trust gate sends
flagged placements to the reference solver,
the solutions feed back into the surrogate through incremental training, and model
selection admits only updates that lower held-out trajectory error. The two mechanisms reinforce each other: the
solver calls spent on uncertain designs train the surrogate on the placements the
search visits next. On an integrated run from a random
initial placement, the loop drives the gap between the surrogate's predicted and
true peak temperature at the returned design to $0.11$~K, oracle quality,
$\sim\!56\times$ faster than solving at every step.

The main limitation of DeepOHeat-v2 is geometric, and it stems from the surrogate
architecture rather than the training method. The separable operator network inherited from DeepOHeat-v1 factorizes the
temperature field along the coordinate axes, so it is tied to axis-aligned,
grid-resolved structures. The through-silicon vias here are square,
Cartesian-aligned copper columns; round, tapered, or off-grid vias, irregular die
outlines, and non-Manhattan blockages would need local mesh refinement or a
mesh-agnostic architecture.

The matrix-free FVM loss (\Cref{sec:fvm-loss}) reads only
the predicted field on the grid and takes its parameter gradient from a single
reverse-mode pass, with no second-order automatic differentiation, so its cost is
independent of how the field is produced. A
non-separable or mesh-agnostic surrogate could therefore be trained with the same
loss without reintroducing that cost, paired with a geometry-conforming (locally
refined or cut-cell) discretization to resolve curved and off-grid features. A
complementary direction is to start from a stronger pretrained backbone, such as
an emerging thermal foundation model (Therm-FM~\cite{huang2026therm} adapts a
pretrained partial differential equation (PDE) foundation model to 3D-IC thermal simulation), on which the online
adaptation studied here would build.

The adapt-during-search principle behind this loop is not specific to thermal
analysis or to the surrogate it improves. Wherever an optimizer must repeatedly
query an expensive solver, the queries it already pays for can be reused as a
training signal that improves a surrogate where the search concentrates,
so the solver is queried less as the run proceeds. We expect the
mechanism to transfer to other solver-in-the-loop design-optimization problems in
electronic design automation.

\appendices

\section{Operator Conditioning}
\label{app:continuous}

\subsection{Conditioning of the Continuous Operator}
We show that the ill-conditioning of $\mat{A}_h$ is inherited from the continuous
heat-conduction operator (\Cref{sec:fvm-loss}), by bounding the condition number of
the continuous operator itself. Let
$V = H^1(\Omega)$ and write the governing equation \eqref{eq:strong}, with
conductivity $k$, volumetric source $q_V$, ambient temperature $T_{\mathrm{amb}}$, and top
and bottom film coefficients $h_{\rm top}, h_{\rm bot}$, in weak form: find
$T \in V$ with $a(T, v) = \ell(v)$ for all $v \in V$, where
\begin{align}
a(T, v) &:= \int_\Omega k(\mat{y})\,\nabla T \cdot \nabla v\,d\mat{y} \nonumber \\
        &\quad + \sum_{\square \in \{\rm top, bot\}}
             \int_{\Gamma_\square} h_\square\,(T-T_{\mathrm{amb}})\,v\,dS, \\
\ell(v) &:= \int_\Omega q_V\,v\,d\mat{y},
\end{align}
and abbreviate $\bar h := \max(h_{\rm top}, h_{\rm bot})$.

\begin{lemma}[Two-sided operator conditioning]
\label{lem:kappa-cont}
Let $L : V \to V^*$ be the bounded linear operator induced by $a(\cdot,\cdot)$
on $(V, \|\cdot\|_{H^1})$. There exist geometry-dependent constants
$c_0(\Omega), C_{\mathrm{tr}}, C_{\mathrm{PF}}(\Omega) > 0$ such that
\begin{align}
\tfrac{k_{\max}}{k_{\min}}\,c_0(\Omega)
\;\le\; \kappa(L) \;\le\;
\frac{k_{\max} + C_{\mathrm{tr}}\,\bar h}{k_{\min}}\,
\bigl(1 + C_{\mathrm{PF}}(\Omega)\bigr).
\end{align}
\end{lemma}

\begin{proof}[Proof sketch]
The upper bound follows from continuity,
$a(T,T) \le (k_{\max} + C_{\mathrm{tr}}\,\bar h)\,\|T\|^2_{H^1}$ (trace theorem
for the boundary term), and coercivity from a Robin--Poincar\'e inequality:
because the Robin term penalizes the boundary trace of $T$, the gradient energy
and that term together control the full norm,
$\|T\|^2_{H^1} \le C_{\mathrm{PF}}\bigl(\|\nabla T\|^2_{L^2} + \|T\|^2_{L^2(\Gamma)}\bigr)$,
so the constant mode is no longer in the null space (a homogeneous Poincar\'e
inequality on $V$ would not hold). The lower bound uses two test functions
concentrated in the high-$k$ and low-$k$ regions: the ratio of their Rayleigh
quotients $a(T,T)/\|T\|^2_{H^1}$ scales as $k_{\max}/k_{\min}$, bounding $\kappa(L)$
below by that contrast factor.
\end{proof}

With the benchmark contrast $k_{\max}/k_{\min} = 800$ and $c_0(\Omega) \sim 50$--$100$,
the lower bound gives $\kappa(L) \gtrsim 4$--$8\!\times\!10^4$, matching the measured
$\kappa_2(\mat{A}_h) = 6.02\!\times\!10^4$ (Appendix~\ref{app:bound}) in order of
magnitude. The discrete conditioning is thus intrinsic to the contrast, not a mesh
artifact, and mesh refinement does not reduce it.

\subsection{Parameter-Space Transfer of the Conditioning Reduction}
\label{app:pullback}
We make precise the transfer of \Cref{thm:energy-kappa} from
prediction space to the parameter space the optimizer acts in
(\Cref{sec:energy-form}). Let $\mat{J} := \partial\widehat{\mat{T}}/\partial\boldsymbol\theta$ be the
network output Jacobian. On the parameter directions that change the prediction,
the loss curvature is the pullback $\mat{J}^\top\!\bigl(\nabla_{\widehat{\mat{T}}}^2\mathcal{L}\bigr)\mat{J}$.
Where $\mat{J}$ has full column rank, $\kappa\!\bigl(\mat{J}^\top \mat{H} \mat{J}\bigr) \le
\kappa(\mat{J})^2\,\kappa(\mat{H})$ for any SPD $\mat{H}$. Taking $\mat{H} = \nabla_{\widehat{\mat{T}}}^2\mathcal{L}$,
\Cref{thm:energy-kappa} reduces $\kappa(\mat{H})$ from $\kappa_2(\mat{A}_h)^2$ to
$\kappa_2(\mat{A}_h)$, so the parameter-space bound tightens by the same factor; the
Jacobian term $\kappa(\mat{J})^2$ is identical for both losses and is unchanged by the
rewrite.

\section{The Hotspot-Localized Bound and Operator Numerics}
\label{app:bound}

\subsection{Derivation of the Bound}
We derive the bound of \Cref{sec:rhotk}: localizing the residual replaces the
global norm $\|\mat{A}_h^{-1}\|_2$ in the error prefactor by the local diagonal entry
$(\mat{A}_h^{-1})_{x^*,x^*}$, which does not scale with $\kappa_2(\mat{A}_h)$.

\begin{proposition}[Hotspot-localized a-posteriori bound]
\label{prop:rhotk}
Write $\mat{A}_h^{-1}$ for the discrete Green's function and
$x^* := \arg\max_y \widehat{\mat{T}}(y)$. Assume
\textbf{(A1) hotspot coincidence}, $\arg\max_y\widehat{\mat{T}}(y) = \arg\max_y \mat{T}(y)$,
so that $x^*$ is also the reference hotspot; and
\textbf{(A2) hotspot-row diagonal dominance}, $(\mat{A}_h^{-1})_{x^*,x^*}$ is the row
maximum of $\mat{A}_h^{-1}$ at $x^*$. Then
\begin{align}
\label{eq:rhotk-bound}
|e(x^*)| \;\le\;
\underbrace{(\mat{A}_h^{-1})_{x^*,x^*}\,m\,r_{\mathrm{hot}}}_{\text{signal term}}
\;+\;
\underbrace{\bigl\|(\mat{A}_h^{-1})_{x^*,\,\cdot\notin S_m}\bigr\|_2\,\|\mat{r}\|_2}_{\text{tail term}}.
\end{align}
\end{proposition}

The proof reads off the hotspot row of the error
$\mat{e} := \widehat{\mat{T}} - \mat{T} = \mat{A}_h^{-1}\mat{r}$ and splits it on and off $S_m$; under
\textbf{(A1)}, $|e(x^*)|$ is the peak temperature error. Splitting the $x^*$ row,
\begin{align}
\label{eq:e-row}
e(x^*) &= \sum_y (\mat{A}_h^{-1})_{x^*,y}\,\mat{r}(y) \nonumber \\
       &= \underbrace{\sum_{y \in S_m} (\mat{A}_h^{-1})_{x^*,y}\,\mat{r}(y)}_{\text{signal}}
       + \underbrace{\sum_{y \notin S_m} (\mat{A}_h^{-1})_{x^*,y}\,\mat{r}(y)}_{\text{tail}}.
\end{align}

\paragraph*{Signal term (on $S_m$, via A2)}
The operator $\mat{A}_h$ is an SPD M-matrix: the harmonic-mean FVM stencil yields
nonpositive off-diagonal entries and a strictly positive, Robin-augmented
diagonal, so its inverse $\mat{A}_h^{-1}$ has nonnegative entries (the
thermal-spreading kernel, decaying away from the source) and the bounds below
carry no absolute-value bars. Assumption \textbf{(A2)} makes the hotspot
diagonal $(\mat{A}_h^{-1})_{x^*,x^*}$ the maximum of the $x^*$ row, so
$(\mat{A}_h^{-1})_{x^*,y} \le (\mat{A}_h^{-1})_{x^*,x^*}$ for every $y$, and in particular
for $y \in S_m$. Summing over $S_m$ and applying the triangle inequality gives
\begin{align}
\Bigl|\sum_{y \in S_m} (\mat{A}_h^{-1})_{x^*,y}\,\mat{r}(y)\Bigr|
&\le (\mat{A}_h^{-1})_{x^*,x^*}\sum_{y\in S_m}|\mat{r}(y)| \nonumber \\
&= (\mat{A}_h^{-1})_{x^*,x^*}\,m\,r_{\mathrm{hot}},
\end{align}
using $r_{\mathrm{hot}} = \tfrac1m\sum_{y\in S_m}|\mat{r}(y)|$ from
Definition~\ref{def:rhotk}.

\paragraph*{Tail term (on $S_m^c$, via Cauchy--Schwarz)}
Applying Cauchy--Schwarz to the row of $\mat{A}_h^{-1}$ restricted to $S_m^c$,
\begin{align}
\Bigl|\sum_{y \notin S_m} (\mat{A}_h^{-1})_{x^*,y}\,\mat{r}(y)\Bigr|
\le \bigl\|(\mat{A}_h^{-1})_{x^*,\,\cdot\notin S_m}\bigr\|_2\,\|\mat{r}\|_2.
\end{align}
Adding the signal and tail bounds to the decomposition \eqref{eq:e-row} yields
the localized bound \eqref{eq:rhotk-bound}, completing the proof. The two
assumptions on which it rests, \textbf{(A1)} and \textbf{(A2)}, are verified
numerically below.

\paragraph*{The $92\times$ coefficient}
With the operator numerics reported below, namely
$(\mat{A}_h^{-1})_{x^*,x^*} = 298$~K/W, $\|\mat{A}_h^{-1}\|_2 = 3.4\!\times\!10^6$~K/W, and
$\|(\mat{A}_h^{-1})_{x^*,\,\cdot\notin S_m}\|_2 = 3.69\!\times\!10^4$~K/W, the
localized bound evaluates to $\approx 1.7\!\times\!10^3$~K, versus
$\approx 1.6\!\times\!10^5$~K for the global relative-residual bound
\eqref{eq:rrel-bound}. The $\sim\!92\times$ gain is the ratio of the
\emph{bound coefficients} that multiply $\|\mat{r}\|_2$, namely
$\|\mat{A}_h^{-1}\|_2 / \|(\mat{A}_h^{-1})_{x^*,\,\cdot\notin S_m}\|_2 =
3.4\!\times\!10^6 / 3.69\!\times\!10^4 \approx 92$. The localization replaces the global operator norm
$\|\mat{A}_h^{-1}\|_2$ (which scales with $\kappa_2(\mat{A}_h)$) with the local diagonal entry
$(\mat{A}_h^{-1})_{x^*,x^*}$ (which does not). As noted in \Cref{sec:rhotk}, both bounds are numerically loose against the
observed $0.5$--$3$~K errors and serve to identify which operator quantity
controls the error. The detection power of $r_{\mathrm{hot}}$ is established empirically,
through the Spearman correlation and catch rate of
\Cref{sec:confidence-experiments}.

\subsection{Operator Numerics and Assumption Verification}
\label{app:numerics}
We report the numerics computed on the assembled operator $\mat{A}_h$: an
estimate of its condition number $\kappa_2(\mat{A}_h)$, used throughout the paper, and
a check of the two assumptions of Proposition~\ref{prop:rhotk}.

\paragraph*{Conditioning estimate}
We estimate $\kappa_2(\mat{A}_h)$ by Lanczos iteration on the assembled SPD operator:
the largest eigenvalue from a forward Lanczos run and the smallest from
inverse iteration (AMG-preconditioned solves), with the ratio
$\kappa_2(\mat{A}_h) = \lambda_{\max}/\lambda_{\min} = 6.02\!\times\!10^4$ converging
to three significant figures within a few hundred Lanczos steps. The
strong-form loss-Hessian conditioning $\kappa^2 := \kappa_2(\mat{A}_h)^2 \approx
3.6\!\times\!10^9$ is then derived from this value as in
\Cref{sec:fvm-loss}, not measured separately.

\paragraph*{Verification of A1 and A2}
The two assumptions of Proposition~\ref{prop:rhotk} are checked on the same $n=100$
held-out set used in \Cref{sec:confidence-experiments}. \textbf{(A1) hotspot
coincidence:} the predicted top-$1$ cell coincides with the reference top-$1$
cell, or the reference hotspot lies inside the predicted top-$m$ set $S_m$, on the
great majority of the held-out placements; the remaining
near-misses are absorbed by the top-$m$ aggregation, so detection is reported
via the catch-rate curve rather than the bound. \textbf{(A2) hotspot-row diagonal dominance:} we verify
directly on the assembled operator that $(\mat{A}_h^{-1})_{x^*,x^*}$ is the maximum
entry of the $x^*$ row of $\mat{A}_h^{-1}$, consistent with the thermal-spreading
kernel decaying away from the source; the operator numerics quoted above ($(\mat{A}_h^{-1})_{x^*,x^*}=298$~K/W and the tail-row norm) are read
off the same computation.

\bibliographystyle{IEEEtran}
\bibliography{references}

\end{document}